\documentclass{article}

\usepackage{microtype}
\usepackage{graphicx}
\usepackage{subcaption}
\usepackage{booktabs} 

\usepackage{hyperref}

\usepackage[preprint]{icml2026}

\usepackage{amsmath}
\usepackage{amssymb}
\usepackage{mathtools}
\usepackage{amsthm}
\renewcommand{\cite}{\citep}
\usepackage[T1]{fontenc}    
\usepackage{hyperref}       
\usepackage{url}            
\usepackage{booktabs}       
\usepackage{arydshln}
\usepackage{amsfonts}       
\usepackage{nicefrac}       
\usepackage{microtype}      
\usepackage{xcolor}         
\definecolor{citecolor}{HTML}{0071BC}
\definecolor{linkcolor}{HTML}{D32F2F}
\definecolor{cellcolor}{HTML}{E3F2FD}
\definecolor{red}{HTML}{D32F2F}
\definecolor{magenta}{HTML}{D81B60}
\usepackage[normalem]{ulem}
\usepackage{svg}
\svgsetup{inkscapearea=page}
\usepackage{amsmath}
\usepackage{amssymb}
\usepackage{mathtools}
\usepackage{amsthm}
\usepackage{amsmath, amssymb, amsthm}
\usepackage{amsfonts}
\usepackage[capitalize,noabbrev]{cleveref}

\theoremstyle{plain}
\newtheorem{theorem}{Theorem}[section]

\theoremstyle{definition}

\newtheorem{hypothesis}{Hypothesis}

\newtheorem{prediction}{Prediction}

\theoremstyle{remark}

\usepackage[textsize=tiny]{todonotes}
\usepackage{microtype}
\usepackage{graphicx}
\usepackage{booktabs} 

\usepackage{pgfplots}
\pgfplotsset{compat = newest}
\usepackage{multirow}
\usepackage{enumitem}
\usepackage{caption}
\usepackage{fancybox}
\usepackage{framed}
\usepackage{tcolorbox}

\usepackage{algorithm}
\usepackage{mathrsfs}
\usepackage{mathtools}
\usepackage{cuted}
\usepackage{pifont}
\usepackage{textcomp}
\usepackage{algorithmic}
\usepackage{tcolorbox}
\usepackage{listings}
\usepackage{makecell}
\usepackage{colortbl}
\usepackage{color}
\usepackage{tcolorbox}
\usepackage{wrapfig}
\usepackage{cancel}
\usepackage{soul,xcolor}
\usepackage{pifont}
\usepackage{tikz}
\usetikzlibrary{shapes, positioning, arrows.meta, calc, decorations.pathmorphing, quotes}
\usepackage{todonotes}
\tcbuselibrary{breakable}
\usepackage{hyperref}
\usepackage{makecell}
\usepackage{url}
\hypersetup{colorlinks=true, linkcolor=linkcolor, citecolor=citecolor,urlcolor=magenta}
\usepackage{booktabs}
\usepackage{hyperref}
\usepackage{url}
\usepackage{graphicx}
\usepackage{subcaption}
\usepackage{amsthm}

\usepackage{enumitem}
\setlist[itemize]{leftmargin=2em}
\setlist[enumerate]{leftmargin=2em}

\DeclareMathOperator{\clip}{clip}

\usepackage[capitalize,noabbrev]{cleveref}

\theoremstyle{plain}

\usepackage[textsize=tiny]{todonotes}

\icmltitlerunning{Distribution-Centric Policy Optimization}

\begin{document}

\twocolumn[
  \icmltitle{Distribution-Centric Policy Optimization \\ Dominates Exploration-Exploitation Trade-off}

    \icmlsetsymbol{equal}{*}
    \icmlsetsymbol{star}{$^\diamond$}
    \icmlsetsymbol{corre}{$\dagger$}
    
    \begin{icmlauthorlist}
        
        \icmlauthor{Zhaochun Li}{bit,zgca,equal}
        \icmlauthor{Chen Wang}{zgca,nku,star,equal}
        \icmlauthor{Jionghao Bai}{zgca,zju,equal}
        \\
        \vspace{1mm}
        \icmlauthor{Shisheng Cui}{bit,zgca,corre}
        \icmlauthor{Ge Lan}{nku}
        \icmlauthor{Zhou Zhao}{zju}
        \icmlauthor{Yue Wang}{zgca,corre}
        
    \end{icmlauthorlist}
    
    \icmlaffiliation{nku}{College of Software, Nankai University,}
    \icmlaffiliation{zgca}{Zhongguancun Academy,}
    \icmlaffiliation{bit}{School of Automation
        , Beijing Institute of Technology,}
    \icmlaffiliation{zju}{College of Computer Science and Technology, Zhejiang University}

    \icmlcorrespondingauthor{Shisheng Cui}{css@bit.edu.cn}
    \icmlcorrespondingauthor{Yue Wang}{yuewang@bjzgca.edu.cn}

  \icmlkeywords{Machine Learning, ICML}
  \vskip 0.3in
]

\printAffiliationsAndNotice{\icmlEqualContribution}

\begin{abstract}
  The exploration–exploitation (EE) trade-off is a central challenge in reinforcement learning (RL) for large language models (LLMs). With Group Relative Policy Optimization (GRPO), training tends to be exploitation driven: entropy decreases monotonically, samples convergence, and exploration fades. Most existing fixes are \textbf{sample-centric}: they seek or bonus rare samples, assuming exploration comes from novel trajectories and tokens. These heuristics depend on the “luck” of informative samples, lack principled control of the policy, and often yield limited or inconsistent gains. In this work, we are the first to introduce a \textbf{distribution-centric} perspective for RL, in which exploration is always guided by a “better” target distribution, and reveal that a policy’s ability to resist entropy collapse is governed by the distribution itself rather than individual samples. Building on this insight, we propose Distribution-Centric Policy Optimization (DCPO), which reformulates entropy regulation as distribution-level regularization. DCPO achieves controllable entropy fully on-policy without sampling from external distributions, enabling efficient exploration while maintaining training stability. Across multiple models and seven benchmarks, DCPO improves over GRPO by about 20\% on average. Overall, DCPO replaces sample-level heuristics with distribution-level principles, offering a theoretically grounded and flexible framework for controllable exploration and a stronger EE trade-off. The code is available in \url{https://github.com/597358816/DCPO}.
\end{abstract}

\section{Introduction}
\textbf{Reinforcement Learning (RL)} has become a cornerstone for enhancing the reasoning capabilities of large language models (LLMs) \cite{glm2024chat, touvron2023llama, schulman2017proximal, rafailov2023direct, zhong2024dpo, wang2024comprehensive}. In complex domains such as mathematics and code generation, recent progress has been driven by Reinforcement Learning with Verifiable Rewards (RLVR), where models learn from outcome-level, automatically checkable rewards \cite{lambert2024tulu, wen2025reinforcement}. In this context, \textbf{exploration} aims to expose the model to as diverse a set of samples as possible during training, so as to discover better solutions across a broader search space, rather than repeatedly circling within a narrow region of behaviors~\cite{sutton1998reinforcement, auer2002finite, strehl2008analysis, kolter2009near}. \textbf{Entropy}, as a quantitative measure of model's uncertainty, has become a practical proxy for a policy's exploration capacity~\cite{schulman2017equivalence, haarnoja2018soft, nachum2017bridging}.

Among RLVR algorithms, Group-Relative Policy Optimization (GRPO) is particularly attractive due to its efficiency, optimizing directly from final trajectory rewards \cite{shao2024deepseekmath, liu2024deepseek, guo2025deepseek}. However, GRPO is widely regarded as exploitation-driven: its training dynamics often reduce entropy monotonically, causing sample convergence and distributional sharpening that progressively suppress exploration \cite{yu2025dapo, cui2025entropy}. When a policy undergoes entropy collapse, its exploration space becomes severely constrained, reducing subsequent learning to mere distributional sharpening within a narrow region of the solution space, which is consistent with recent findings that RL for LLMs often fails to expand the boundary of reasoning capability~\cite{yue2025does}. To mitigate this issue, prior variants introduce explicit entropy encouragement, such as Entropy-Reg \cite{o2016pgq, hou2025advancing} and Entropy-Adv \cite{cheng2025reasoning}. More recently, AEPO proposes to inject exploration by applying a REINFORCE regularization on high-temperature samples, which alleviates entropy collapse without introducing large optimization bias \cite{wang2025arbitrary}. Yet these approaches remain \textbf{sample-centric}: they focus on obtaining or bonusing rarer samples, assuming that exploration arises from individual trajectories and tokens with high novelty. As a consequence, sample-centric optimization relies on the “luck” of drawing sufficiently informative instances, lacks principled control over the policy distribution, and offers limited insight into the overall optimization dynamics, resulting in limited or inconsistent gains.

This motivates a \textbf{distribution-centric} perspective, which views exploration as an intrinsic property of the entire policy distribution and leverages a ``better'' target distribution to guide exploration, rather than relying on fortuitous rare samples. To understand the underlying mechanism, we conduct a series of importance-sampling-based analyses using AEPO as a representative baseline and evaluate policies by their ability to resist entropy collapse. We reveal and prove that the key to avoiding entropy collapse is not the novelty of individual samples, but the expectation gradient induced by the distribution. Building on this insight, we propose \emph{Distribution-Centric Policy Optimization} (DCPO), which reformulates entropy regulation as a distribution-level regularization problem.
The key design components of DCPO are as follows:
\begin{itemize}[leftmargin=1.5em]
    \item \textbf{Fully Online and On-Policy for Exploration.} DCPO performs optimization entirely under the current policy, eliminating off-distribution sampling and improving both training efficiency and stability.
    \item \textbf{REINFORCE as Regularization.} DCPO employs the REINFORCE policy gradient as a regularization term, ensuring entropy regulation toward the target distribution.
    \item \textbf{Double Importance Sampling.} DCPO introduces two layers of importance weighting—one for correcting the deviation between the sampling distribution and the current policy with clipping for stable training, and another for adjusting the gradient expectation toward the target distribution—thereby enabling precise distribution-level regularization.
\end{itemize}
Experiments show that DCPO outperforms GRPO by an average of about 20\% across seven widely used reasoning benchmarks, and surpasses the strongest existing entropy-control baseline AEPO. DCPO provides compelling evidence that the essence of exploration lies in the distribution itself, rather than the diversity of sampled instances. \textbf{In scenarios that demand exploration, seeking a better distribution offers a more efficient path than blindly searching through the vast textual space.} 

Our main contributions are summarized as follows:
\begin{itemize}[leftmargin=1.5em]
    \item \textbf{Distribution-centric optimization.} We reinterpret entropy regulation as a property of the policy’s sampling distribution within the policy gradient expectation, providing a principled explanation for how distributional diversity governs entropy dynamics.
    \item \textbf{DCPO for entropy regulation.} We propose DCPO, which achieves entropy control using only samples from the original distribution. It offers higher optimization performance and efficiency.
    \item \textbf{A principled framework for EE trade-off.} Beyond a single algorithm, we establish a unified, distribution-grounded perspective that formalizes the Exploration–Exploitation balance as a controllable optimization process. Within this framework, we identify the \textbf{Precision–Prediction (PP) Trade-off}, which characterizes how precise sampling ensures stability, while predictive importance sampling enhances optimization effectiveness. This principle provides the foundation for a broader family of EE trade-offs in RL.
\end{itemize}

\section{Preliminary}
Our work focuses on fine-tuning LLM using RL for tasks with verifiable solutions, such as mathematical reasoning and code generation. Verifiable rewards remove the traditional reward model used in RL and instead assign binary 0/1 rewards. To be specific,for a given query $q$ and its corresponding reference response $o^{*}$, the reward for any response $o$ sampled from policy $\pi_{\theta}$ is defined as \[R(q,o) =\textbf{1}[o=o^{*}].\]
In this paper,  an auto-regressive language model parameterized by $\theta$ is defined as a policy $\pi_{\theta}$. 
Suppose the LLM is a softmax policy, that is
\begin{align*}
    \pi_{\theta}(o_{t}|q_{t})=\frac{\text{exp}(l(q_{t},o_{t}))}{\sum_{o'_{t}}\text{exp}(l(q_{t},o'_{t}))},
\end{align*}
where $q_{t}$ is the concatenation of query q followed by $o_{<t}$, and $l(q_{t},o_{t})$ is the logit of token $o_{t}$ given input $q_{t}$. Furthermore, given a temperature T, we define:
\begin{align*}
    \pi_{\theta}^{T}(o_{t}|q_{t})=\frac{\text{exp}(l(q_{t},o_{t})/T)}{\sum_{o'_{t}}\text{exp}(l(q_{t},o'_{t})T)}.
\end{align*}
\subsection{Policy optimization}
\textbf{REINFORCE} is a cornerstone of policy gradient methods that directly optimizes the policy $\pi_{\theta}$ by maximizing the expected reward over sampled trajectories. The objective is formally defined as:
\begin{align*}
    &\mathcal{J} _{\text{REINFORCE}}(\theta) =  \mathbb{E}_{q \sim P(Q), \; o_i \sim \pi_{\theta_{\text{old}}}(O|q)} 
          \frac{1}{|o_i|} 
         \\  &\sum_{t=1}^{|o_i|} 
            \min \Big[ r_{i,t}(\theta)\,R(q,o),\; 
            \text{clip}\!\big(r_{i,t}(\theta), 1-\epsilon, 1+\epsilon \big)\,R(q,o) \Big].
\end{align*}
A key characteristic of REINFORCE is that the reward $R$ is applied uniformly as a credit to all actions (tokens) within the trajectory. To reduce the high variance inherent in this estimator, a baseline $b$ is typically subtracted from the reward, leading to the more common form $R(q,o)-b$. While fundamentally important, the high variance of the REINFORCE estimator and its reliance on full trajectory rewards make it challenging to apply directly in large-scale settings without further stabilization mechanisms.

\textbf{GRPO} bypasses the need for the value model by computing the relative advantage of each response within a group of responses to the same query. Specifically, GRPO optimizes the following objective:
\begin{align*}
    \mathcal{J} &_{\text{GRPO}}(\theta) =  \mathbb{E}_{q \sim P(Q), \; \{o_i\}_{i=1}^G \sim \pi_{\theta_{\text{old}}}(O|q)} 
         \frac{1}{G} \sum_{i=1}^G \frac{1}{|o_i|} 
         \\  &\sum_{t=1}^{|o_i|} 
            \min \Big[ r_{i,t}(\theta)\,\hat{A}_{i,t},\; 
            \text{clip}\!\big(r_{i,t}(\theta), 1-\epsilon, 1+\epsilon \big)\,\hat{A}_{i,t} \Big].
\end{align*}
where $G$ is the number of generated responses to each query $q$, and the importance
 ratio $r_{i,t}(\theta)$ and advantage $\hat{A}_{i,t}$ of token $o_{i,t}$ are:
 $r_{i,t}(\theta) = \frac{\pi_{\theta}(o_{i,t}|q,o_{i,<t})}{\pi_{\theta_{old}}(o_{i,t}|q,o_{i,<t})}$ ,
 $\hat{A}_{i,t} = \hat{A}_{i}=\frac{R(q,o_{i})-\text{mean}(\{R(q,o_{j})\}_{j=1}^{G})}{\text{std}(\{R(q,o_{j})\}_{j=1}^{G})}$
 respectively, where all the tokens in $o_{i}$ share the same advantage as $\hat{A}_{i}$.
 
\subsection{Policy entropy}
Policy entropy quantifies the predictability or randomness inherent in the actions selected by an agent. Given a query $q$, let $o$ denote a response sampled from policy model $\pi_{\theta}$ for query $q$. For each token $o_{t}$ in response $o$, we denote the token-level entropy as: 
\begin{align*}
      \mathcal{H}_{t}(\pi_{\theta}):=-\mathbb{E}_{o_{t} \sim \pi_{\theta}(\cdot|q,o_{<t})}\big[\log\pi_{\theta}(o_{t}|q,o_{<t})\big],
  \end{align*}
and then we can further denote policy entropy as:
\begin{align*}
    \mathcal{H}(\pi_{\theta}):=\mathbb{E}_{q \sim P(Q), \; o \sim \pi_{\theta}(O|q)}\frac{1}{|o|}\sum_{t=1}^{|o|}\mathcal{H}_{t}(\pi_{\theta}).
\end{align*}

Such entropy quantifies the uncertainty level of the policy on current prompts and is widely adopted in maximum entropy RL as a regularization term. Furthermore, prior work discovers a relationship between REINFORCE updates and policy entropy:
\begin{theorem}[REINFORCE entropy relationship]
\label{theo:REIN-ent}
High-temperature REINFORCE (Eq.~\eqref{eq:REIN-ent}, $T>1$) induces a relative increase in policy entropy, while low-temperature REINFORCE (Eq.~\eqref{eq:REIN-ent}, $T<1$) induces a relative decrease in policy entropy.

\vspace{-12pt}
\begin{equation}
\label{eq:REIN-ent}
\begin{split}
         &\mathcal{J}(\theta) =  \mathbb{E}_{q \sim P(Q), \; o_i \sim \textcolor{red}{\pi^{T}_{\theta_{\text{old}}}(O|q)}}
          \frac{1}{|o_i|} 
         \\  &\sum_{t=1}^{|o_i|} 
            \min \Big[ r_{i,t}(\theta)\,R(q,o),\; 
            \clip\!\big(r_{i,t}(\theta), 1-\epsilon, 1+\epsilon \big)\,R(q,o) \Big].
\end{split}
\end{equation}
\vspace{-4pt}

\end{theorem}
See AEPO \cite{wang2025arbitrary} for a formal statement and proof; here we use the intuitive form as preliminary background.

\begin{table*}[t]
    \small
    \centering
    \caption{Comparison of different loss configurations designed to disentangle and verify the respective roles of samples and distributions in exploration.}
    \label{tab:loss_ablation}

    \setlength{\abovedisplayskip}{0pt}
    \setlength{\belowdisplayskip}{0pt}

    \begin{tabular}{r @{ \ = \ } l r}
        \toprule
        \addlinespace 
        \multicolumn{3}{l}{
            $\ \rho_{i,t}= \pi_{\theta_{\text{old}}}^{T}(o_{i,t}|q,o_{i,<t})/\pi_{\theta_{\text{old}}}(o_{i,t}|q,o_{i,<t})$,\quad
            $r_{i,t}(\theta) = \pi_{\theta}(o_{i,t}|q,o_{i,<t})/\pi_{\theta_{old}}(o_{i,t}|q,o_{i,<t})$, 
        }\\
        \addlinespace 
        \multicolumn{3}{l}{
            \ $T = T_{\text{low}} +
            \big(T_{\text{high}} - T_{\text{low}}\big) 
            \, \mathbf{1}\!\left[\,\mathcal{H}(\pi_{\theta_{\text{old}}}) < \mathcal{H}_0\,\right]$.
        }\\
        \addlinespace 
        \midrule
        \addlinespace 
        $\displaystyle \mathcal{J}_{1}(\theta)$ & 
        $\displaystyle \mathcal{J}_{\text{GRPO}}(\theta) + \alpha \mathbb{E}_{q \sim P(Q), \{o_i\}_{i=1}^G \sim \textcolor{red}{\pi_{\theta_{\text{old}}}}(O|q)}
        \frac{1}{|o_i|} \sum_{t=1}^{|o_i|} \min \Big[ r_{i,t}(\theta) R(q, o_i), \text{clip}(r_{i,t}(\theta), 1-\epsilon, 1+\epsilon) R(q, o_i) \Big]$. &
        \refstepcounter{equation}(\theequation)\label{eq:form1} \\ 
        
        \addlinespace 

        $\displaystyle \mathcal{J}_{2}(\theta)$ & 
        $\displaystyle \mathcal{J}_{\text{GRPO}}(\theta) + \alpha \mathbb{E}_{q \sim P(Q), \{o_i\}_{i=1}^G \sim \textcolor{red}{\pi^T_{\theta_{\text{old}}}}(O|q)}
        \frac{1}{|o_i|} \sum_{t=1}^{|o_i|} \min \Big[ r_{i,t}(\theta) R(q, o_i), \text{clip}(r_{i,t}(\theta), 1-\epsilon, 1+\epsilon) R(q, o_i) \Big]$. &
        \refstepcounter{equation}(\theequation)\label{eq:form2} \\

        \addlinespace
        $\displaystyle \mathcal{J}_{3}(\theta)$ & 
        $\displaystyle \mathcal{J}_{\text{GRPO}}(\theta) + \alpha \mathbb{E}_{q \sim P(Q), \{o_i\}_{i=1}^G \sim \textcolor{red}{\pi_{\theta_{\text{old}}}}(O|q)}
        \frac{1}{|o_i|} \sum_{t=1}^{|o_i|} \min \Big[ \textcolor{red}{\rho_{i,t}} \cdot r_{i,t}(\theta) R(q, o_i), \textcolor{red}{\rho_{i,t}} \cdot \text{clip}(r_{i,t}(\theta), 1-\epsilon, 1+\epsilon) R(q, o_i) \Big]$. &
        \refstepcounter{equation}(\theequation)\label{eq:form3} \\

        \addlinespace

        $\displaystyle \mathcal{J}_{4}(\theta)$ & 
        $\displaystyle \mathcal{J}_{\text{GRPO}}(\theta) + \alpha \mathbb{E}_{q \sim P(Q), \{o_i\}_{i=1}^G \sim \textcolor{red}{\pi^T_{\theta_{\text{old}}}}(O|q)}
        \frac{1}{|o_i|} \sum_{t=1}^{|o_i|} \min \Big[ \textcolor{red}{\rho_{i,t}^{-1}} \cdot r_{i,t}(\theta) R(q, o_i), \textcolor{red}{\rho_{i,t}^{-1}} \cdot \text{clip}(r_{i,t}(\theta), 1-\epsilon, 1+\epsilon) R(q, o_i) \Big]$. &
        \refstepcounter{equation}(\theequation)\label{eq:form4} \\
        \addlinespace
        \bottomrule
    \end{tabular}
\end{table*}
  
\subsection{Importance sampling}
Importance sampling (IS) estimates an expectation under a target distribution $p(x)$ using samples from a proposal distribution $q(x)$. For an arbitrary measurable function $f$, the relationship is formally expressed as:
\begin{align*}
    \mathbb{E}_{x \sim p}[f(x)] = \mathbb{E}_{x \sim q}[\frac{p(x)}{q(x)} \cdot f(x)],
\end{align*}

When applying IS to auto-regressive language models, one could treat the whole generated sequence as the random variable, whose probability factorizes as $p(o|q) = \Pi_{t=1}^{|o|} \pi(o_t|q, o_{<t})$. The resulting trajectory-level importance weight is a product of token-level ratios and is known to suffer from severe variance for long sequences, making estimates numerically unstable.

To ensure viable optimization, it is standard to avoid explicit trajectory-level weights and instead apply IS correction at the token level:
\vspace{-4pt}
\begin{equation}
    \label{eq:IS}
    \begin{split}
    \mathbb{E}_{o \sim \pi_{\theta}(O|q)}[r(o)] &=\mathbb{E}_{o\sim \pi'(O|q)}\frac{p_{\theta}(o|q)}{p'(o|q)}\cdot r(o)
    \\& \approx \mathbb{E}_{o \sim \pi'(O|q)}[\sum_{t=1}^{|o|}\frac{\pi_{\theta}(o_{t}|q,o_{<t})}{\pi'(o_{t}|q,o_{<t})}\cdot r(o)].
    \end{split}
\end{equation}
\vspace{-4pt}

This token-level form avoids multiplicative accumulation of ratios and yields more stable gradient estimates in practice.

\section{Method}
\label{sec:method}

This section develops our method from first principles. We begin by formulating two mutually exclusive hypotheses on the mechanism behind entropy regulation in AEPO and Theorem \ref{theo:REIN-ent}. We then design targeted empirical analyses to adjudicate between these hypotheses and identify the true driver of entropy control. Finally, guided by the validated mechanism, we present \emph{Distribution-Centric Policy Optimization} (DCPO), which achieves controllable entropy in a fully on-policy manner from a distribution-centric perspective.

\subsection{Sample or distribution}
\label{sec:sample-or-distribution}

We consider two competing explanations for why temperature-based evaluation can improve exploration. The first is \emph{sample-centric}, attributing entropy gains to the increased chance of drawing rare ``critical'' samples. The second is \emph{distribution-centric}, arguing that entropy regulation emerges from the structure of the distribution through its induced expected gradient. We state both hypotheses below and test which mechanism governs entropy control.

\begin{hypothesis}
    \label{hypo:sample}
    \textbf{Sample-Centric optimization dominates the EE trade-off.}
\end{hypothesis}

Hypothesis \ref{hypo:sample} posits that entropy control stems from its function as a superior sampling strategy. Under this view, the high-temperature distribution $\pi_{\theta}^{T}$ ($T>1$) is effective because it more frequently discovers and samples certain critical instances that possess a potent, intrinsic potential to boost entropy. The role of $\pi_{\theta}^{T}$, therefore, is merely to increase the prevalence of these "good" samples, with the resulting entropy increase being driven by the inherent properties of the samples themselves. 

\vspace{2mm}
\begin{hypothesis}
    \label{hypo:distribution}
    \textbf{Distribution-Centric optimization dominates the EE trade-off.}
\end{hypothesis}

\begin{figure*}[t]
    \begin{equation}
        \label{eq:DCPO}
        \begin{aligned}
            \mathcal{J}_{GRPO}(\theta) 
            &= \; \mathbb{E}_{q \sim P(Q), \; \{o_i\}_{i=1}^G \sim \pi_{\theta_{\text{old}}}(O|q)} 
            \frac{1}{G} \sum_{i=1}^G \frac{1}{|o_i|} \sum_{t=1}^{|o_i|} 
            \min \Big[ r_{i,t}(\theta)\,\hat{A}_{i,t},\; 
            \text{clip}\!\big(r_{i,t}(\theta), 1-\epsilon, 1+\epsilon \big)\,\hat{A}_{i,t} \Big],
            \\[2mm]
            \mathcal{J}_{\mathrm{DCPO}}(\theta)
            &=
            \mathbb{E}_{(q,a)\sim\mathcal{D},\,O=\{o_i\}_{i=1}^{G}\sim\pi_{\theta_{\mathrm{old}}}(\cdot|q)}\\
            &
            \frac{1}{G}\sum_{i=1}^{G}\frac{1}{|o_i|}\sum_{t=1}^{|o_i|}
            \min\!\Big[
            r_{i,t}(\theta)\,\Big(\hat{A}_t+\textcolor{red}{\alpha\rho_{i,t} R(o_i)}\Big),
            \;
            \mathrm{clip}\!\big(r_{i,t}(\theta),1-\epsilon,1+\epsilon\big)\,\Big(\hat{A}_t+\textcolor{red}{ \alpha \rho_{i,t} R(o_i)}\Big)\Big], 
        \end{aligned}
    \end{equation}

    
    \noindent \textit{where 
        $\rho_{i,t}= \frac{\pi_{\theta_{\text{old}}}^{T}(o_{i,t}|q,o_{i,<t})}{\pi_{\theta_{\text{old}}}(o_{i,t}|q,o_{i,<t})}$, $r_{i,t}(\theta) = \frac{\pi_{\theta}(o_{i,t}|q,o_{i,<t})}{\pi_{\theta_{\text{old}}}(o_{i,t}|q,o_{i,<t})}$
        and \ $T = 
        T_{\text{low}} +
        \big(T_{\text{high}} - T_{\text{low}}\big) 
        \, \mathbf{1}\!\left[\,\mathcal{H}(\pi_{\theta_{\text{old}}}) < \mathcal{H}_0\,\right]$.}
\end{figure*}

Hypothesis \ref{hypo:distribution} posits that the efficacy of entropy control is an emergent property of the intrinsic mathematical structure of the distribution. Under this view, the specifics of any individual sample are secondary. The decisive factor is the macroscopic behavior of the expected gradient $\mathbb{E}_{\tau \sim \pi_{\theta}^{T}}[R(\tau)\cdot\nabla_{\theta}\text{log}\pi_{\theta}(\tau)]$. The structure of $\pi_{\theta}^{T}$, as a target distribution, systematically and inherently biases the expected gradient $\mathbb{E}_{\tau \sim \pi_{\theta}^{T}}[R(\tau)\cdot\nabla_{\theta}\text{log}\pi_{\theta}(\tau)]$ toward a direction that increases policy entropy, irrespective of whether the constituent samples are "critical" or "commonplace".

\subsection{From sample to distribution}

To adjudicate between the sample-centric and distribution-centric hypotheses, we design comparative experiments that independently manipulate (i) the \emph{samples} and (ii) the distribution under which the regularizer's expected gradient is computed. 
For brevity, we only present the regularization term $\mathcal{J}_{reg}(\theta)$; the full objective is $\mathcal{J}(\theta)=\mathcal{J}_{\text{GRPO}}(\theta)+\alpha\,\mathcal{J}_{reg}(\theta)$.

\begin{enumerate}[leftmargin=1.5em]
    \item {$\mathcal{J}_{1}$ (Eq.~\ref{eq:form1}):} Standard REINFORCE regularization. Both samples and target distribution are the original policy $\pi_{\theta_{\text{old}}}$. In practice, this loss can not escape from entropy collapse during training.
    \item {$\mathcal{J}_{2}$ (Eq.~\ref{eq:form2}):} The objective of AEPO. Both samples and target distributions are the temperature-scaled policy $\pi_{\theta_{\text{old}}}^{T}$. Empirically, this loss enables entropy regulation and mitigates entropy collapse during training.
    \item {$\mathcal{J}_{3}$ (Eq.~\ref{eq:form3}):} Sample from $\pi_{\theta_{\text{old}}}$, but apply token-level IS weights $\rho_{i,t}=\pi_{\theta_{\text{old}}}^{T}(o_{i,t}|q,o_{i,<t})/\pi_{\theta_{\text{old}}}(o_{i,t}|q,o_{i,<t})$ to simulate distribution under $\pi_{\theta_{\text{old}}}^{T}$. This isolates the effect of a ``better'' \emph{evaluation distribution}.
    \item {$\mathcal{J}_{4}$ (Eq.~\ref{eq:form4}):} Sample from $\pi_{\theta_{\text{old}}}^{T}$, but apply inverse weights $\rho_{i,t}^{-1}$ to recover evaluation under $\pi_{\theta_{\text{old}}}$. This isolates the effect of ``better'' \emph{samples}.
\end{enumerate}

Together, $\{\mathcal{J}_1,\mathcal{J}_2,\mathcal{J}_3,\mathcal{J}_4\}$ form a controlled grid over (samples, distribution), where IS is used to decouple the two. This enables falsifiable predictions for the two hypotheses.

\begin{theorem}
    \label{theo:unbias}
    Assuming token-level IS is unbiased, let $\rho_{i,t}= \frac{\pi_{\theta_{\text{old}}}^{T}(o_{i,t}|q,o_{i,<t})}{\pi_{\theta_{\text{old}}}(o_{i,t}|q,o_{i,<t})}$, the following equation is an unbiased estimation of Eq.~\eqref{eq:REIN-ent}.
    \begin{equation}
        \begin{split}
        &\mathbb{E}_{q \sim P(Q), \{o_i\}_{i=1}^G \sim \textcolor{red}{\pi_{\theta_{\text{old}}}}(O|q)}\\
        & \qquad
        \frac{1}{|o_i|} \sum_{t=1}^{|o_i|} \min \Big[ \textcolor{red}{\rho_{i,t}} \cdot r_{i,t}(\theta) R(q, o_i), \\
        & \qquad\qquad\textcolor{red}{\rho_{i,t}} \cdot \clip(r_{i,t}(\theta), 1-\epsilon, 1+\epsilon) R(q, o_i) \Big].
        \end{split}
    \end{equation}
\end{theorem}
\begin{proof}
The result follows directly from the Eq.~\eqref{eq:IS}. 
Under the assumption that token-level IS is unbiased, replacing the $\pi_{\theta_{\text{old}}}^{T}$ in Eq.~\eqref{eq:REIN-ent} by an token-wise IS $\rho_{i,t}$ yields an unbiased estimator of the original objective. 
\end{proof}

\begin{prediction}
If Hypothesis \ref{hypo:sample} holds, then $\mathcal{J}_4$ can achieve entropy control, while $\mathcal{J}_3$ will reach entropy collapse.
\end{prediction}

Under the Sample-Centric Hypothesis, entropy control is attributed to the availability of more exploratory trajectories. Therefore, any variant that \emph{samples} from $\pi_{\theta_{\text{old}}}^{T}$ (i.e., $\mathcal{J}_2$ and $\mathcal{J}_4$) should regulate entropy, whereas variants sampling from $\pi_{\theta_{\text{old}}}$ (i.e., $\mathcal{J}_1$ and $\mathcal{J}_3$) should collapse.

\begin{prediction}
If Hypothesis \ref{hypo:distribution} holds, then $\mathcal{J}_3$ can achieve entropy control, while $\mathcal{J}_4$ will reach entropy collapse.
\end{prediction}

This prediction follows directly from Theorem~\ref{theo:unbias}. Specifically, $\mathcal{J}_3$ uses IS to \emph{unbiasedly evaluate} the regularizer under $\pi_{\theta_{\text{old}}}^{T}$ despite sampling from $\pi_{\theta_{\text{old}}}$; hence it should retain entropy control if the distribution is the driving factor. Conversely, $\mathcal{J}_4$ explicitly importance-samples the gradient back to $\pi_{\theta_{\text{old}}}$ via $\rho_{i,t}^{-1}$ (Theorem~\ref{theo:unbias}), and thus should lose the ability to regulate entropy if entropy control is a distribution-level property.

\begin{figure}[t]
    \centering
    \includegraphics[width=1\linewidth]{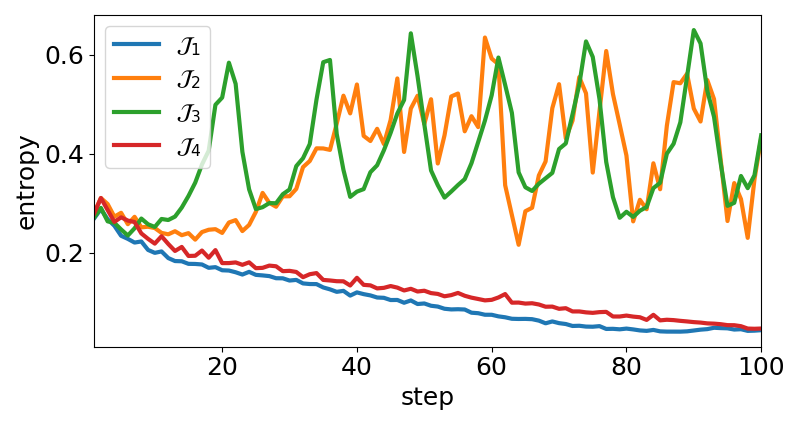}
    \caption{$\mathcal{J}_3$ successfully regulates entropy, while $\mathcal{J}_4$ leads to entropy collapse.}
    \label{fig:j3j4}
\end{figure}

Fig.~\ref{fig:j3j4} compares $\mathcal{J}_3$ and $\mathcal{J}_4$ and shows that only $\mathcal{J}_3$ maintains stable policy entropy, whereas $\mathcal{J}_4$ collapses. This outcome supports Hypothesis \ref{hypo:distribution}: sustained exploration is governed by the \emph{distribution under which the expected gradient is evaluated}, rather than by the incidental presence of a few high-entropy samples. Equivalently, entropy control emerges as a property of distribution-centric optimization.

\begin{tcolorbox}[colback=blue!5, colframe=blue!70] \textit{\textbf{Distribution-Centric Optimization Dominates the Exploration–Exploitation Trade-off.}} \end{tcolorbox}

\begin{table*}[t]
 \caption{Main results on seven math-reasoning benchmarks. DCPO consistently outperforms GRPO and entropy-based baselines, achieving the best average score and the largest gains on difficult contest benchmarks.}

 \label{tab:main_results_all}
 \centering
 \resizebox{\textwidth}{!}{%
  \renewcommand{\arraystretch}{1.4}
  \begin{tabular}{
    >{\raggedright\arraybackslash}m{2.8cm}  
    *{7}{>{\centering\arraybackslash}m{1.6cm}}| 
    >{\centering\arraybackslash}m{2.4cm}    
   }
   \toprule
   \textbf{Benchmarks} & \textbf{AIME24$_{\times32}$} & \textbf{AIME25$_{\times32}$} & \textbf{\ AMC$_{\times32}$}  & \textbf{GSM8K} & \textbf{MATH} & \textbf{Minerva} & \textbf{Olympiad} & \textbf{Average} \\
   \midrule
   \rowcolor{gray!10} Qwen2.5-7B  & 7.91 & 5.31 & 36.2 & 88.5 & 64.4 & 22.0 & 29.3 & 36.24 \\
   \quad +GRPO  & 17.1 & 7.60 & 65.8 & 92.3 & 75.6 & 36.8 & 38.8 & 47.70\\
   \quad +Entropy-Reg  & 13.6 & 8.85 & 67.4 & 92.3 & 76.8 & 35.5 & 39.1 & 47.65\\
   \quad +Entropy-Adv  & 14.8 & 8.23 & 67.3 & 91.9 & 76.6 & 38.2 & 37.5 & 47.79\\
   \quad +AEPO & 17.5 & 11.4 & 69.3 & 92.9 & 78.0 & 37.8 & 40.2 & 49.57 \\
    \rowcolor{blue!5}
    \quad +DCPO & \textbf{18.8} & \textbf{15.3} & \textbf{69.9} & \textbf{93.0} & \textbf{79.2} & \textbf{38.2} & \textbf{42.2} & \textbf{50.94} \\
    \noalign{\vspace{-1.5mm}}
    \rowcolor{blue!5} \qquad $\Delta$ vs. GRPO
    & \textcolor{green!70!black}{\textbf{+1.7}}
    & \textcolor{green!70!black}{\textbf{+7.7}}
    & \textcolor{green!70!black}{\textbf{+4.1}}
    & \textcolor{green!70!black}{\textbf{+0.7}}
    & \textcolor{green!70!black}{\textbf{+3.6}}
    & \textcolor{green!70!black}{\textbf{+1.4}}
    & \textcolor{green!70!black}{\textbf{+3.4}}
    & \textcolor{green!70!black}{\textbf{+3.24 (+28.3\%)}} \\
   \midrule
   \rowcolor{gray!10} Qwen2.5-math-7B  & 15.5 & 7.81 & 42.1 & 65.4 & 59.4 & 11.0 & 26.7 & 32.56 \\
   \quad +GRPO  & 32.1 & 11.0 & 72.4 & 88.7 & 80.6 & 34.6 & 41.8 & 51.60 \\
   \quad +Entropy-Reg  & 31.4 & 10.1 & 74.3 & 87.0 & 80.4 & 35.7 & 40.4 & 51.10 \\
   \quad +Entropy-Adv  & 32.1 & 11.4 & 72.1 & 87.8 & 80.4 & 37.5 & 42.1 & 51.76 \\
   \quad +AEPO
    & \textbf{36.4} & 12.6 & 74.8 & 89.5 & 81.6 & 39.0 & 43.0 & 53.87 \\
   \rowcolor{blue!5}
    \quad +DCPO
    & \textbf{35.2} & \textbf{17.8} & \textbf{76.3} & \textbf{92.0} & \textbf{82.0} & \textbf{43.4} & \textbf{43.7} & \textbf{55.77} \\
    \noalign{\vspace{-1.5mm}}
    \rowcolor{blue!5}
    \qquad $\Delta$ vs. GRPO
    & \textcolor{green!70!black}{\textbf{+3.1}}
    & \textcolor{green!70!black}{\textbf{+6.8}}
    & \textcolor{green!70!black}{\textbf{+3.9}}
    & \textcolor{green!70!black}{\textbf{+3.3}}
    & \textcolor{green!70!black}{\textbf{+1.4}}
    & \textcolor{green!70!black}{\textbf{+8.8}}
    & \textcolor{green!70!black}{\textbf{+1.9}}
    & \textcolor{green!70!black}{\textbf{+4.17 (+21.9\%)}} \\
    \bottomrule
   \textbf{Benchmarks} & \textbf{AIME24$_{\times32}$} & \textbf{AIME25$_{\times32}$} & \textbf{HMMT25$_{\times32}$} & \textbf{Minerva} & \textbf{Olympiad} & \textbf{GPQA$_{\text{diamond}}$} & \textbf{MMLU$_{\text{pro}}$} & \textbf{Average} \\
   \midrule
   \rowcolor{gray!10} Qwen3-4B  & 36.4 & 22.7 & 13.0 & 42.3 & 47.2 & 6.06 & 72.67 & 34.33 \\
   \quad +GRPO & 52.9 & 41.5 & 27.1 & 46.7 & 60.0 & 10.1 & 74.1 & 44.63\\
    \quad +Entropy-Reg  & 52.4 & 42.6 & 25.3 & 46.3 & 60.1 & 10.6 & 74.1 & 44.48 \\
    \quad +Entropy-Adv & 51.6 & 41.7 & 25.5 & 46.0 & 58.4 & 10.6 & 74.8 & 44.08 \\
    \quad +AEPO & 54.5 & \textbf{43.7} & 26.3 & 47.8 & 60.9 & 10.6 & 73.9 & 45.31\\
    \rowcolor{blue!5}
    \quad +DCPO & \textbf{56.6} & 42.7 & \textbf{28.8} & \textbf{48.2} & \textbf{61.4} & \textbf{11.1} & \textbf{76.1} & \textbf{46.41}\\
    \noalign{\vspace{-1.5mm}}
    \rowcolor{blue!5}
    \qquad $\Delta$ vs. GRPO
    & \textcolor{green!70!black}{\textbf{+3.7}}
    & \textcolor{green!70!black}{\textbf{+1.2}}
    & \textcolor{green!70!black}{\textbf{+1.7}}
    & \textcolor{green!70!black}{\textbf{+1.5}}
    & \textcolor{green!70!black}{\textbf{+1.4}}
    & \textcolor{green!70!black}{\textbf{+1.0}}
    & \textcolor{green!70!black}{\textbf{+2.0}}
    & \textcolor{green!70!black}{\textbf{+1.78 (+17.2\%)}} \\
   \bottomrule
  \end{tabular}
 }
\end{table*}

\begin{table}[t]
 \caption{Pass@128 results on contest benchmarks with Qwen3-4B. DCPO consistently improves GRPO and entropy-based baselines.}
 \label{tab:pass128}
 \centering
 \resizebox{0.48\textwidth}{!}{%
  \renewcommand{\arraystretch}{1.4}
  \begin{tabular}{
    >{\raggedright\arraybackslash}m{2.5cm}  
    *{3}{>{\centering\arraybackslash}m{1.7cm}} 
   }
   \toprule
   \textbf{Pass@128} & \textbf{AIME24} & \textbf{AIME25} & \textbf{HMMT25}  \\
   \midrule
   \rowcolor{gray!10} Qwen3-4B  & 76.7 & 63.3 & 60.0  \\
   \quad +GRPO & 83.3 & 76.7 & 66.7 \\
    \quad +Entropy-Reg  & 83.3 & 73.3 & 66.7  \\
    \quad +Entropy-Adv & 83.3 & 73.3 & 70.0  \\
    \quad +AEPO & 83.3 & 76.7 & 66.7 \\
    \rowcolor{blue!5}
    \quad +DCPO & \textbf{86.7} & \textbf{80.0} & \textbf{73.3} \\
   \bottomrule
  \end{tabular}
 }
\end{table}

\begin{table*}[h]
\centering
\caption{Different exploration levels in DCPO on Qwen2.5-Math-7B. Moderate exploration yields the best overall performance.}
\label{tab:entropy_ablation}
 \resizebox{\textwidth}{!}{%
  \renewcommand{\arraystretch}{1.4}
  \begin{tabular}{
    >{\raggedright\arraybackslash}m{3.2cm}  
    *{7}{>{\centering\arraybackslash}m{1.6cm}}| 
    >{\centering\arraybackslash}m{1.6cm}    
   }
   \toprule
   \textbf{Benchmarks} & \textbf{AIME24} & \textbf{AIME25} & \textbf{AMC}  & \textbf{GSM8K} & \textbf{MATH} & \textbf{Minerva} & \textbf{Olympiad} & \textbf{Average} \\
   \midrule
   \rowcolor{gray!10} Qwen2.5-math-7B  & 15.5 & 7.81 & 42.1 & 65.4 & 59.4 & 11.0 & 26.7 & 37.66 \\
   \rowcolor{blue!5}\quad +DCPO $\mathcal{H}_0=0.25$  & 35.2 & 17.8 & \textbf{76.3} & \textbf{92.0} & \underline{82.0} & \textbf{43.4} & 43.7 & \textbf{55.77}\\
   \rowcolor{blue!5}\quad +DCPO $\mathcal{H}_0=0.50$  & \textbf{37.0} & \underline{18.1} & \underline{73.8} & \underline{91.6} & \textbf{82.6} & \underline{39.7} & \textbf{44.1} & \underline{55.27} \\
   \rowcolor{blue!5}\quad +DCPO $\mathcal{H}_0=0.75$ & \underline{35.5} & \textbf{19.4} & 70.4 & 91.1 & 81.2 & 37.9 & \underline{43.9} & 54.21 \\
   \rowcolor{blue!5}\quad +DCPO $\mathcal{H}_0=1.00$  & 32.8 & 13.4 & 68.0 & 90.6 & 80.6 & 38.6 & 43.0 & 52.43 \\
   \bottomrule
  \end{tabular}
 }
\end{table*}

\subsection{Distribution-Centric Policy Optimization}
The comparative experiments above indicate that entropy regulation is driven primarily by the \emph{distribution} shaping the expected gradient, rather than by the chance of sampling a few ``good'' exploratory trajectories. 

Based on this finding, we introduce \textbf{Distribution-Centric Policy Optimization (DCPO)}, which is obtained by taking Eq.~\eqref{eq:DCPO} and combining like terms into a compact objective. Instead of betting on lucky samples, DCPO continuously constructs a virtual target distribution that is preferable to the current policy for exploration (in our case, the distribution with high temperature and entropy is more diverse), and consistently uses this virtual distribution to guide optimization. This provides a persistent exploratory signal and helps the policy escape local optima.

\begin{itemize}[leftmargin=1.5em]
\item \textbf{Fully Online and On-Policy:}
DCPO alleviates entropy collapse in a fully online and on-policy manner, drawing all samples from the current policy. This removes distribution mismatch between sampling and updates, yielding lower-bias gradients and more stable learning dynamics. Consequently, DCPO can maintain stable entropy regulation without entropy collapse, supporting effective (near-optimal) exploration.
    
\item \textbf{REINFORCE as Regularization:}
According to Theorem~\ref{theo:REIN-ent}, DCPO leverages REINFORCE as a \emph{regularization} term to mitigate entropy collapse: it provides an unbiased mechanism to suppress negatively rewarded samples while preserving  optimization towards the target distribution. Moreover, although REINFORCE gradients can be high-variance, their impact is effectively controlled in DCPO because the regularizer is scaled by a small coefficient $\alpha$.
    
\item \textbf{Double Importance Sampling:}
DCPO uses two importance ratios with distinct roles. The first ratio $r_{i,t}(\theta)$ corrects the behavior policy toward the updated online policy, with clipping to control variance and stabilize learning. The second ratio $\rho_{i,t}$ adjusts the \emph{expectation of the regularizer's gradient} toward the virtual target distribution (e.g., a higher-entropy distribution), thereby providing a distribution-centric exploratory signal while maintaining on-policy optimization.
\end{itemize}

\section{Experiments}
Full details of the experimental setup (models, datasets, benchmarks, and implementation) are provided in Appendix~\ref{app:exp_setup}. In brief, we fine-tune Qwen2.5-7B \cite{yang2024qwen2}, Qwen2.5-Math-7B \cite{yang2024qwen2}, and Qwen3-4B \cite{yang2025qwen3} on DAPO-17K \cite{yu2025dapo}.

\subsection{Main Results}
Table~\ref{tab:main_results_all} reports our main results across three backbones on our seven-benchmark suites. 
Across all settings, \textbf{DCPO achieves the best average performance} and consistently improves over GRPO and entropy-based baselines, supporting that \emph{distribution-centric} entropy control is an effective mechanism for stabilizing exploration and improving reasoning accuracy.

Overall, DCPO improves the average score over GRPO by \textbf{+3.24} on Qwen2.5-7B, \textbf{+4.17} on Qwen2.5-Math-7B, and \textbf{+1.78} on Qwen3-4B. 
These correspond to \textbf{+28.3\%}, \textbf{+21.9\%}, and \textbf{+17.2\%} additional gains relative to GRPO's improvement over the corresponding base models. 
Notably, the gains concentrate on the hardest contest-style benchmarks. For Qwen2.5-7B, DCPO yields \textbf{+7.7} on AIME25 and \textbf{+4.1} on AMC over GRPO, and improves Olympiad by \textbf{+3.4}. 
For Qwen2.5-Math-7B, DCPO provides a large boost on Minerva (\textbf{+8.8} over GRPO) and further improves AIME25 by \textbf{+6.8}, while also improving GSM8K (\textbf{+3.3}) without sacrificing stability. 
On Qwen3-4B, DCPO continues to improve difficult benchmarks such as AIME24 (\textbf{+3.7}), HMMT25 (\textbf{+1.7}), and Olympiad (\textbf{+1.4}), and also brings consistent gains on broader evaluations (e.g., \textbf{+2.0} on MMLU$_{\text{pro}}$ and \textbf{+1.0} on GPQA$_{\text{diamond}}$).

We further evaluate high-budget sampling performance with Pass@128 on contest benchmarks using Qwen3-4B (Table~\ref{tab:pass128}). 
DCPO consistently yields the best results, improving GRPO from 83.3 to \textbf{86.7} on AIME24, from 76.7 to \textbf{80.0} on AIME25, and from 66.7 to \textbf{73.3} on HMMT25. 
These gains suggest that DCPO not only improves typical decoding performance but also raises the solution upper bound under larger sampling budgets, consistent with its goal of maintaining an exploratory optimization distribution.
\begin{figure}[t]
    \centering
    \includegraphics[width=0.9\linewidth]{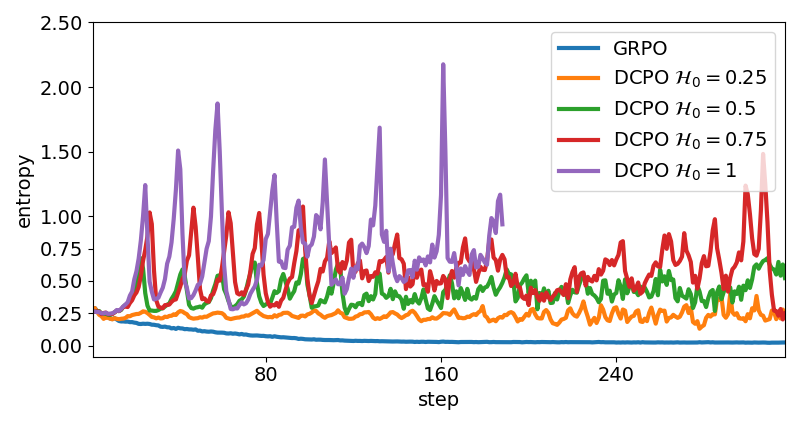}
    \caption{Training entropy
     of DCPO on different exploration levels.}
    \label{fig:AEPO2-main}
\end{figure}

\begin{table*}
    \caption{Key component ablations of DCPO on Qwen2.5-Math-7B. Removing either double importance sampling or the REINFORCE term causes entropy collapse and severely degrades performance.}
    \label{tab:module_ablation}
    \resizebox{\textwidth}{!}{
    \renewcommand{\arraystretch}{1.4}
    \begin{tabular}{
    >{\centering\arraybackslash}m{0.85\textwidth} 
    >{\centering\arraybackslash}m{0.15\textwidth}    
   }
    \toprule
    \textbf{Ablation loss and performance} & \textbf{Entropy control} \\
    \midrule
    \noalign{\vspace{0.5mm}}
    \multicolumn{2}{c}{
    $\displaystyle 
    \begin{aligned}
    \mathcal{J}_{\text{w/o double IS}} 
    &= \; \mathbb{E}_{(q,a),O\sim\pi_{\theta_{\mathrm{old}}}}
            \frac{1}{G}\sum_{i=1}^{G}\frac{1}{|o_i|}\sum_{t=1}^{|o_i|}
            \min\!\Big[
            r_{i,t}(\theta)\,\Big(\hat{A}_t+\textcolor{red}{\alpha R(o_i)}\Big),
            \;
            \mathrm{clip}\!\big(r_{i,t}(\theta),1-\epsilon,1+\epsilon\big)\,\Big(\hat{A}_t+\textcolor{red}{\alpha  R(o_i)}\Big)\Big]
    \end{aligned}$} \refstepcounter{equation}(\theequation)\label{eq:ablation1} 
    \\
    \cmidrule(l{1em}r){1-2}
    \rowcolor{red!5}
    \resizebox{0.85\textwidth}{!}{
    \begin{tabular}{
    *{7}{>{\centering\arraybackslash}m{1.4cm}}| 
    >{\centering\arraybackslash}m{2.4cm}|    
   }
    AIME24 & AIME25 & AMC & GSM8K & MATH & Minerva & Olympiad & Avg \\
    30.2 & 11.6 & 73.6 & 87.7 & 80.0 & 36.0 & 40.6 & 51.39 \textcolor{red}{(-4.38)} \\
    \end{tabular}} &  Entropy collapse \\
    \midrule

    \multicolumn{2}{c}{
    $\displaystyle 
    \begin{aligned}
    \mathcal{J}_{\text{w/o REINFORCE}} 
    &= \; \mathbb{E}_{(q,a),\,O\sim\pi_{\theta_{\mathrm{old}}}}
            \frac{1}{G}\sum_{i=1}^{G}\frac{1}{|o_i|}\sum_{t=1}^{|o_i|}
            \min\!\Big[
            r_{i,t}(\theta)\,\textcolor{red}{(1+\alpha\rho_{i,t})\hat{A}_t},
            \;
            \mathrm{clip}\!\big(r_{i,t}(\theta),1-\epsilon,1+\epsilon\big)\,\textcolor{red}{(1+\alpha\rho_{i,t})\hat{A}_t}\Big]
    \end{aligned}$} \refstepcounter{equation}(\theequation)\label{eq:ablation2} 
    \\
    \cmidrule(l{1em}r){1-2}
    \rowcolor{red!5}
    \resizebox{0.85\textwidth}{!}{
    \begin{tabular}{
    *{7}{>{\centering\arraybackslash}m{1.4cm}}| 
    >{\centering\arraybackslash}m{2.4cm}|    
   }
    AIME24  & AIME25 & AMC & GSM8K & MATH & Minerva & Olympiad & Avg \\
    30.9 & 11.1 & 73.9 & 87.9 & 81.2 & 36.8 & 41.3 & 51.87 \textcolor{red}{(-3.90)} \\
    \end{tabular}} &  Entropy collapse \\
    \bottomrule
    \end{tabular}
}
\end{table*}

\subsection{Ablation Study}
We ablate DCPO from two complementary angles: (i) the \emph{degree of exploration} controlled by the target entropy level $\mathcal{H}_0$, and (ii) the \emph{key components} that enable exploration.

\textbf{Exploration level.} Fig.~\ref{fig:AEPO2-main} and Table~\ref{tab:entropy_ablation} varies $\mathcal{H}_0$ on Qwen2.5-Math-7B. We find that moderate exploration performs best overall: $\mathcal{H}=0.25$ achieves the highest average score, while increasing $\mathcal{H}_0$ beyond 0.5 gradually reduces the average. Although higher $\mathcal{H}_0$ can improve certain difficult contest benchmarks (e.g., AIME25 peaks at 19.4 with $\mathcal{H}=0.75$), overly aggressive exploration weakens overall optimization.

\textbf{Component ablations.} Table~\ref{tab:module_ablation} removes either the \textbf{double IS} correction or the REINFORCE term in DCPO. Both ablations lead to entropy collapse and a substantial drop in average performance , i.e., a degradation of \textbf{3.9--4.4} points compared to the full DCPO setting (55.77). This confirms that distribution-centric entropy control requires both (i) IS toward the target distribution as an explicit exploration guide and (ii) a REINFORCE term as a regularizer to induce and sustain exploration.

\section{Discussion}
\subsection{Exploration–Exploitation trade-off}
DCPO and AEPO provides clear evidence that entropy is a decisive factor in the EE trade-off, enabling arbitrary degrees of control over exploratory behavior, whereas GRPO remains strictly exploitation-driven.  By venturing into unfamiliar reasoning spaces during optimization, DCPO ultimately converges to superior test-time reasoning ability.
\subsection{Precision-Prediction trade-off}
Beyond the classical EE trade-off, our results reveal a new dimension of trade-off that emerges when attempting to sustain exploration. Specifically, maintaining exploration requires breaking free from distributional sharpening—meaning that the gradient expectation must incorporate distributional knowledge not present in the original sampling distribution. DCPO and AEPO realize this principle through two theoretically equivalent yet practically distinct approaches:

\begin{itemize}[leftmargin=1.5em]
    \item AEPO directly samples from the target distribution, allowing the algorithm to obtain genuine information from that distribution—potentially even complete information in ideal cases. This enables it to theoretically converge to the target distribution in a deterministic and stable manner. However, sampling from a different distribution is inherently off-policy, which introduces a degree of distribution shift and can consequently impair optimization effectiveness.
    \item DCPO samples from the current policy and uses IS to adjust the gradient expectation toward the target distribution, making the algorithm fully on-policy and thus a more favorable choice in the short term. However, IS cannot fully access the information contained in the target distribution—it only provides an approximation. This approximation inevitably introduces variance, which accumulates over the course of optimization and may lead to instability in the long run. 
\end{itemize}

These characteristics collectively define:

\begin{tcolorbox}[colback=blue!5, colframe=blue!70]
\textit{\textbf{Precision-Prediction (PP) Trade-off}: Relying more on precise sampling improves stability, while relying more on predictive and virtual distribution by IS improves optimization effectiveness — and for any given scenario, an optimal Precision-Prediction balance exists that enables the desired Exploration–Exploitation trade-off.}
\end{tcolorbox}


\section{Conclusion}
We study entropy collapse in RLVR for large language models, where objectives such as GRPO are often exploitation-driven and progressively suppress exploration. We revisit the exploration from a \emph{distribution-centric} perspective and show that resisting entropy collapse is governed by the distribution through its induced expected gradient rather than by rare ``good'' samples. Based on this insight, we propose \emph{Distribution-Centric Policy Optimization} (DCPO), a fully on-policy, distribution-level regularization method. Across multiple backbones and seven-benchmark suites, DCPO consistently outperforms GRPO and entropy-based baselines---especially on harder contest benchmarks---and ablations verify that stable entropy control relies on both target-distribution importance sampling and REINFORCE-as-regularization.

Beyond the specific algorithm, our analyses suggest a broader \textbf{Precision-Prediction trade-off}: \emph{precise} sampling from an exploratory distribution stabilizes entropy regulation, whereas \emph{predicting} the exploratory distribution via importance sampling improves optimization effectiveness. For a given training scenario, an optimal PP balance enables the desired exploration--exploitation (EE) trade-off, and REINFORCE-as-regularization is essential for realizing this balance in RLVR.

Future work includes generalizing the distribution-centric perspective beyond entropy control and the EE trade-off. We aim to establish distribution-centric optimization as a general foundation for controllable, efficient policy learning, not only for exploration, but for a wider class of optimization goals in large-scale RL.

\section*{Declaration of AI}
AI is only used for translation and language polishing in this paper.

\bibliography{reference}
\bibliographystyle{icml2026}

\newpage
\appendix
\onecolumn

\section{Related work}

\subsection{RL for LLM Post-Training and Reasoning}
Reinforcement learning (RL) is a key paradigm for post-training large language models (LLMs) to align with human feedback and task objectives \cite{ouyang2022training}. 
Prior work includes RLHF-style optimization with PPO \cite{openai2023gpt4, team2024gemini1_5, wei2023instructiongpt, liu2023visual, schulman2017proximal} and more efficient preference-based alternatives such as DPO \cite{rafailov2024direct, zhong2024dpo, wang2024comprehensive}. 
For domains with verifiable, rule-based rewards (e.g., math and code), RLVR has enabled strong reasoning systems such as DeepSeek-R1, Kimi k1.5, and Qwen3 \cite{lambert2024tulu, wen2025reinforcement, guo2025deepseekr1, team2025kimi, team2025kimi1_5, yang2025qwen3}. 
Within RLVR, GRPO has become a widely used baseline due to its value-function-free design \cite{shao2024deepseekmath, liu2024deepseek, guo2025deepseek}, but it is often exploitation-driven and prone to entropy collapse \cite{cui2025entropy}.

\subsection{Entropy and Exploration in RL for LLMs}
Entropy is a practical proxy for exploration in policy optimization \cite{sutton1999policy, williams1992simple}. 
In LLM RL, however, naive entropy bonuses are often coarse or unstable under long-horizon generation. Recent work therefore uses entropy more strategically---either by constraining it to selected tokens (e.g., SIREN, AEnt) \cite{jiang2025rethinking, shen2025entropy} or by incorporating it into reward shaping and credit assignment (e.g., GTPO, EGSW, CHORD) \cite{tan2025gtpo, vanlioglu2025entropy, zhang2025policy}; related baselines such as Entropy-Reg and Entropy-Adv further modify the objective with entropy-related terms \cite{o2016pgq, hou2025advancing, cheng2025reasoning}. 
AEPO is a strong recent baseline that applies temperature-adjusted REINFORCE regularization using samples from a temperature-scaled distribution \cite{wang2025arbitrary}. 
Despite these efforts, existing methods remain largely sample-centric and lack principled control at the \emph{policy distribution} level, often yielding limited or inconsistent improvements as training proceeds \cite{yu2025dapo, wang2025beyond, cui2025entropy}.

Our work addresses this gap by advocating a \emph{distribution-centric} perspective: we characterize entropy regulation through the evaluation distribution and its induced expected gradient, and propose DCPO to realize distribution-level regularization in a fully on-policy manner.

\section{Experimental Setup Details}
\label{app:exp_setup}

\subsection{Model and Dataset}
We evaluate DCPO on mathematical reasoning RL fine-tuning. Experiments are conducted on three backbones: Qwen2.5-7B, Qwen2.5-Math-7B \cite{yang2024qwen2}, and Qwen3-4B~\cite{yang2025qwen3}. 
All models are trained on DAPO-17K \cite{yu2025dapo}, which provides RL-oriented question--answer pairs with verifiable rewards.

\subsection{Benchmarks and Metrics}
We evaluate on seven-benchmark suites. For Qwen2.5-7B and Qwen2.5-Math-7B, we report results on AIME24 \cite{hf_aime2024}, AIME25 \cite{balunovic2025matharena}, AMC \cite{lightman2023lets}, GSM8K \cite{cobbe2021gsm8k}, MATH \cite{lightman2023lets}, Minerva Math \cite{lewkowycz2022solving}, and Olympiad \cite{lightman2023lets}. 
For Qwen3-4B, some standard math benchmarks can be relatively easy and may saturate, so we additionally include more challenging evaluations. Specifically, we report results on AIME24 \cite{hf_aime2024}, AIME25 \cite{balunovic2025matharena}, HMMT25 \cite{balunovic2025matharena}, Minerva Math \cite{lewkowycz2022solving}, Olympiad \cite{lightman2023lets}, GPQA${\text{diamond}}$ \cite{rein2023gpqa}, and MMLU${\text{pro}}$ \cite{wang2024mmlupro}.

Since some contest benchmarks contain only a small number of problems, we report \textbf{Avg@32} for contest-style benchmarks (e.g., AIME/AMC/HMMT): for each problem, we sample 32 solutions and average correctness across samples, which provides a more stable estimate than Pass@32 in low-cardinality settings. 
We additionally report \textbf{Pass@128} on selected contest benchmarks for Qwen3-4B.

\begin{table*}[t]
\centering
\caption{Summary of implementation and evaluation details for all compared methods.}
\label{tab:impl_details}
\renewcommand{\arraystretch}{1.4}
\setlength{\tabcolsep}{5pt}
\small
\begin{tabular}{
l
>{\centering\arraybackslash}m{3.2cm}
>{\centering\arraybackslash}m{2.2cm}
>{\centering\arraybackslash}m{2.7cm}
>{\centering\arraybackslash}m{4.2cm}
}
\toprule
\multicolumn{5}{l}{\textbf{RL settings}} \\
\midrule
Hardware & \multicolumn{4}{c}{8$\times$A800 GPUs (40GB)} \\
Policy model init & \multicolumn{4}{c}{Qwen2.5-7B / Qwen2.5-Math-7B / Qwen3-4B} \\
Training dataset & \multicolumn{4}{c}{DAPO-17K} \\
Max response length & \multicolumn{4}{c}{8192} \\
Batch / mini-batch size & \multicolumn{4}{c}{512 / 128} \\
Rollout group size $G$ & \multicolumn{4}{c}{8} \\
Learning rate & \multicolumn{4}{c}{$1\times10^{-6}$} \\
Temperature (training) & \multicolumn{4}{c}{1.0} \\
Clip range $(\epsilon_{\text{low}},\epsilon_{\text{high}})$ & \multicolumn{4}{c}{(0.2, 0.2)} \\
Reward type & \multicolumn{4}{c}{Binary reward} \\
\toprule
\multicolumn{5}{l}{\textbf{Evaluation settings}} \\
\midrule
Max response length & \multicolumn{4}{c}{8192} \\
Top-$p$ (eval) & \multicolumn{4}{c}{0.95} \\
Temperature (eval) & \multicolumn{4}{c}{0.1 for Pass@1;\;\;0.6 for Pass@128} \\
\toprule
\multicolumn{5}{l}{\textbf{Method-specific settings}} \\
\midrule
\textbf{Method} & \textbf{Entropy bonus} & \textbf{$T_{\text{high}}/T_{\text{low}}$} & \textbf{REINFORCE samples} & \textbf{Extra coefficient} \\
\midrule
GRPO & -- & -- & -- & -- \\
Entropy-Reg & $\lambda=0.015$ & -- & -- & -- \\
Entropy-Adv & $(\beta,\kappa)=(0.4,2)$ & -- & -- & -- \\
AEPO & -- & 1.2 / 0.8 & 60 (entropy up), 30 (entropy down) & -- \\
DCPO & -- & 1.2 / 0.8 & -- &
$\displaystyle \alpha=\frac{0.1(\mathcal{H}-\mathcal{H}(\pi_{\theta_{\text{old}}}))}{\text{accuracy rate of batch}}$ \\
\bottomrule
\end{tabular}
\end{table*}

\end{document}